\documentclass{IEEEtran}
\usepackage{mathrsfs}
\usepackage{amsmath}
\usepackage{amssymb}
\usepackage{amsthm}
\ifCLASSINFOpdf \else \fi

\usepackage[dvips]{graphicx}
\usepackage{algorithm}
\usepackage{algorithmic}
\usepackage{slashbox}
\usepackage{subfigure}

\newtheorem{Theo}{Theorem}
\newtheorem{Lemm}{Lemma}

\theoremstyle{definition}

\theoremstyle{Remark}
\newtheorem{remark}{Remark}

\newtheorem{Defi}{Definition}

\usepackage{color}
\usepackage{slashbox}
\usepackage[normalem]{ulem}
\begin{document}

\title{Capacity and Delay of Unmanned Aerial Vehicle Networks with Mobility}

\author{Zhiqing Wei,~Zhiyong Feng,~Haibo Zhou,~Li Wang,~and~Huici Wu
\thanks{
%	Copyright (c) 2012 IEEE. Personal use of this material is permitted. However, permission to use this material for any other purposes must be obtained from the IEEE by sending a request to pubs-permissions@ieee.org.

This work is supported by the National
Natural Science Foundation of China (No. 61631003, No. 61525101, No. 61601055).

Zhiqing Wei and Zhiyong Feng are with
Key Laboratory of Universal Wireless Communications, Ministry of Education,
School of Information and Communication Engineering,
Beijing University of Posts and Telecommunications, Beijing, 100876, China (e-mail: \{weizhiqing, fengzy\}@bupt.edu.cn).

Haibo Zhou is with the School of
Electronic Science and Engineering,
Nanjing University, Nanjing 210023, China
(e-mail: haibozhou@nju.edu.cn, h53zhou@uwaterloo.ca).

Li Wang is with Beijing Key Laboratory of Work Safety
Intelligent Monitoring,
School of Electronic Engineering, Beijing University
of Posts and Telecommunications, Beijing, 100876, China. She
is also with the Key Laboratory of the
Universal Wireless Communications, Ministry of Education, China (e-mail: liwang@bupt.edu.cn).

Huici Wu is with the National
Engineering Lab for Mobile Network Technologies,
Beijing University of Posts and Telecommunications,
Beijing 100876,
China (e-mail: dailywu@bupt.edu.cn).}}

%\markboth{IEEE ,~Vol.~, No.~, ~2011}%
%{Shell \MakeLowercase{\textit{et al.}}: Bare Demo of IEEEtran.cls for Journals}

\maketitle

\begin{abstract}
Unmanned aerial vehicles (UAVs)
are widely exploited in environment monitoring,
search-and-rescue, etc.
However,
the mobility and
short flight duration of UAVs
bring challenges for UAV networking.
In this paper, we study the UAV networks with
$n$ UAVs acting as aerial sensors.
UAVs generally have
short flight duration
and need to frequently get energy replenishment
from the control station.
Hence the returning UAVs
bring the data of
the UAVs along the returning paths to the control station
with a store-carry-and-forward (SCF) mode.
A critical range for the distance between
the UAV and the control
station is discovered.
Within the critical range, the
per-node capacity of the SCF mode is
$\Theta (\frac{n}{{\log n}})$ times higher
than that of the multi-hop mode.
However, the per-node capacity
of the SCF mode outside the critical range
decreases with the distance between the UAV and the control station.
To eliminate the critical range,
a mobility control scheme is proposed
such that the capacity scaling laws
of the SCF mode are the same for all UAVs,
which improves the capacity performance of UAV networks.
Moreover, the delay of the SCF mode is derived.
The impact of the size of the entire region, the velocity of UAVs, the number of UAVs and the flight duration of UAVs on the delay of SCF mode is analyzed.
This paper reveals that the mobility and
short flight duration of UAVs
have beneficial effects
on the performance of UAV networks,
which may motivate the study of SCF schemes for
UAV networks.
\end{abstract}
\IEEEpeerreviewmaketitle

\section{Introduction}

% 第一部分（分为两段）：UAV的发展，以及用途，
% 引出UAV组网（不涉及到具体的技术，这部分review的文献多是survey 和magazine）

UAVs are flexibly deployed in the challenging
environment to accomplish
tasks such as monitoring of air pollution or
toxic gas leakage \cite{UAV_Gas_Detection}.
As shown in Fig. \ref{fig_scenario},
UAVs act
as aerial sensors to collect
the data of air pollution and
transmit data to the control station.
Besides, UAVs can act as
aerial base stations (BSs)
to provide communication services
to the areas with natural disasters,
traffic congestions or concerts \cite{UAV_Disaster_Managaemnt, UAV_aerial_BS}.
UAVs can also act as aerial relays
for vehicular networks
to improve the reliability of wireless
links in vehicular networks \cite{UAV_VANET}.
Compared with single UAV system,
the UAV swarm can complete missions
in a more efficient and economical way \cite{UAV_Survey},
which makes UAV networking an emerging research field.

\begin{figure}[!t]
\centering
\includegraphics[width=0.55\textwidth]{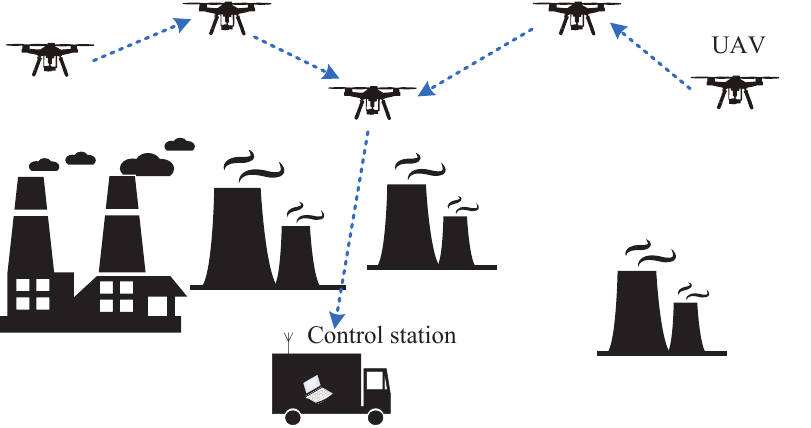}
\caption{The application scenario of UAVs.}
\label{fig_scenario}
\end{figure}

% 第二部分：针对上文提到的无人机的应用，综述UAV组网的研究成果
Generally, the studies of UAV networking mainly focus on
the construction of broadband/robust
UAV networks and the cooperation between
UAV networks and other networks, such as
vehicular networks, satellite networks, cellular networks.
For the construction of broadband/robust
UAV networks,
Xiao \emph{et al.} in \cite{UAV_mmWAVE} incorporated
millimeter-wave into UAV
networks to support high
data rate transmission.
Zhang \emph{et al.} in \cite{Cog_UAV}
studied the
spectrum sharing between UAV networks and ground cellular networks
to improve the spectrum efficiency.
Cai \emph{et al.} in \cite{UAV_protocol}
and Jiang \emph{et al.} in \cite{UAV_MAC_protocol}
designed the medium
access control (MAC) protocol to construct robust UAV
networks.
Sbeiti \emph{et al.} in \cite{UAV_routing_protocol}
studied the routing protocols for airborne mesh networks.
As to the cooperation between UAV networks and ground networks,
Zhou \emph{et al.} in \cite{UAV_Ground_Cooperation} exploited
multiple UAVs to act as aerial BSs
to support vehicular networks.
Sharma \emph{et al.} in \cite{Capacity_Het_UAV_Enhan}
adopted UAVs to enhance the
coverage and capacity of two-tier cellular networks.

In UAV networks, the mobility of UAVs,
the dynamic topology of UAV networks and
the short flight duration
will bring challenges for the design of network protocols
and resource scheduling schemes.
Facing these challenges,
the literatures \cite{UAV_mmWAVE}-\cite{Capacity_Het_UAV_Enhan} studied
the transmission schemes and protocols
of UAVs.
However, the
mobility and fluid topology of UAVs
have also beneficial effects on the
UAV networks with appropriate network controlling and design.
Mozaffari \emph{et al.} in \cite{UAV_D2D}
studied the scenario that a single UAV
acts as a flying BS for
an underlaid ad hoc network,
where the optimal altitude of UAV was derived
to maximize the system sum-rate.
They further optimized the density
and altitude of multiple UAVs to
maximize the coverage probability
of ground networks \cite{UAV_Coverage}.
Fadlullah \emph{et al.} in \cite{UAV_Trajectory} designed
a trajectory control algorithm
to improve end-to-end
connectivity and delay
of UAV assisted wireless networks.
Besides acting as relays or BSs,
UAVs can also act as mobile
sinks to gather
the data of sensor networks,
where appropriate mobility control
of UAVs can be designed to optimize the network performance.
Ergezer \emph{et al.} in \cite{Capacity_UAV_Sensor}
optimized the flight
paths of UAVs to maximize
the amount of collected data.
Say \emph{et al.} in \cite{UAV_Mobile_Sink}
exploited the mobility information of UAV to
assign different priorities to the nodes
within UAV's
coverage area to improve network capacity.
Lyu \emph{et al.} in \cite{UAV_DTN1}
took advantage of the cyclical trajectory of
fixed wing UAV to deliver
the data of
ground nodes.

% 第三部分：重点说一下UAV作为空中传感器的应用，以及研究成果

% 第四部分：稍微Review一下Scaling Law的论文，从容量角度，提一下通过控制无人机移动性，提升容量的方法，然后阐述本文的成果

The literatures \cite{UAV_D2D}-\cite{UAV_DTN1}
mainly focused on the mobility control algorithms to
improve the performance of UAV-assisted networks.
Actually, the intrinsic mobility pattern of UAVs
can also be exploited to improve the performance of UAV networks.
For example, the UAVs generally
have short flight duration
and they will frequently return to the control station
to get energy replenishment, such as charging batteries.
The returning UAVs can
carry data for the UAVs along
the returning paths to the control station with
an SCF transmission manner.
The network capacity can be boosted
with SCF manner \cite{ZWEI}.
However, there are rare studies
exploiting the intrinsic mobility pattern
of UAVs to improve the performance of UAV networks.

\begin{figure*}
\centering
\subfigure[{UAVs are distributed in 3D space.}]{
\label{fig_UAV_Net}
\includegraphics[width=0.48\textwidth]{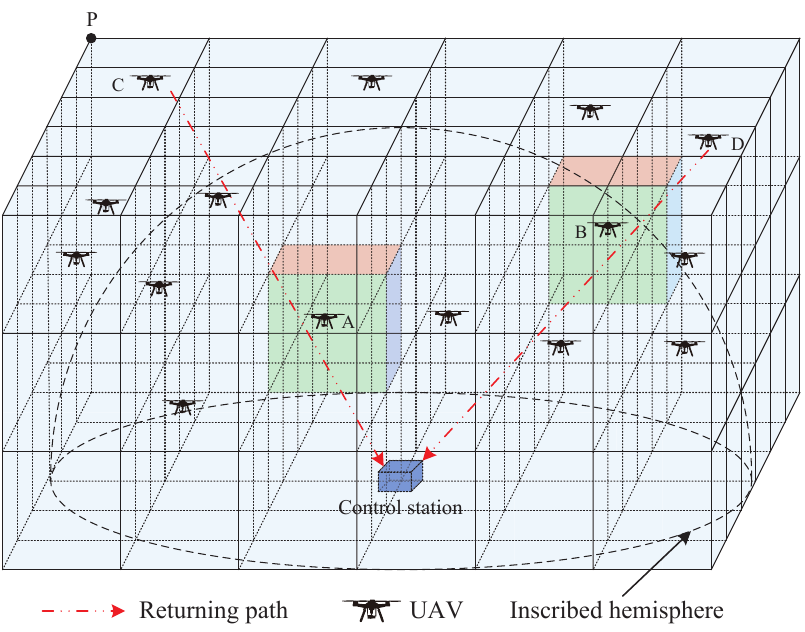}}
\subfigure[{UAVs are distributed in aerial 2D plane.}]{
\label{fig_UAV_Net_2D}
\includegraphics[width=0.48\textwidth]{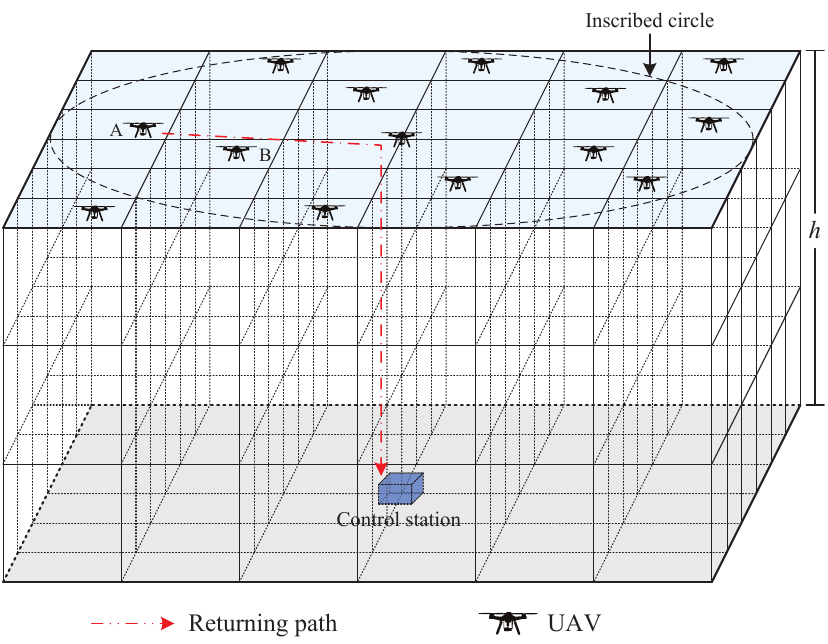}}
\caption{Network model of UAV networks}
\label{fig_network_model}
\end{figure*}

% 第五部分：说一下本文的结构

%介绍本文研究的场景
In this paper, the mobility
of UAVs is exploited to
construct the SCF transmission manner.
The scenario with UAVs acting as
aerial sensors is considered, where
the UAVs are uniformly distributed in a
three-dimensional (3D) space or an aerial two-dimensional (2D)
plane monitoring the environment.
The returning UAVs are exploited to implement the SCF scheme.
The capacity and delay scaling laws of UAV networks are
studied.
It is discovered that the capacity of SCF mode
is $\Theta (\frac{n}{{\log n}})$ times higher
compared with that of multi-hop mode.
The delay of UAV networks with SCF mode is derived.
The impact of the size of the entire region,
the velocity of UAVs, the number of UAVs and the
flight duration of UAVs on the delay of SCF mode is analyzed.
Besides, the critical range is discovered.
The capacity of the UAVs outside the
critical range is
smaller than that within the critical range,
which results in the inhomogeneity
of the capacity of UAV networks.
The per-node capacity of UAV network with SCF mode
is a piecewise function of the distance
between the UAV and the control station with the
critical as a threshold.
Hence a mobility control scheme is
proposed to eliminate the critical
range such that the capacity
scaling laws of all UAVs are the same.
The mobility and the short flight duration are the features of UAVs.
In common sense, the mobility
and the short flight duration of UAVs bring challenges for the
design of network protocols and
resource scheduling schemes. However, the mobility and the
short flight duration can also bring
benefits for UAV networks. The SCF mode proposed in this
paper turns the wasted returning time into treasure and makes
the UAV networks scalable.
%This paper proves the
%advantages of SCF mode in UAV networks, which
%may motivate the design of more
%SCF transmission schemes for UAV networks.

The remainder of this paper is organized as follows.
The system model is introduced in Section II.
In Section III and Section IV, the
capacity and delay scaling laws of 3D
and 2D UAV networks are studied.
In Section V, the main results are discussed.
Finally, this paper is summarized in Section VI.
The key parameters and notations are listed in Table \ref{sys_para}.

\begin{table}[!t]
 \caption{\label{sys_para}Key Parameters and Notations}
 \begin{center}
 \begin{tabular}{l l}
 \hline
 \hline

    {Symbol} & {Description} \\
 %Symbol & Description & Values\\

  \hline

  $n$ & Number of UAVs\\
  $v$ & Velocity of UAV\\
  $t_0$ & Flight duration of UAV\\
  $h$ &  Altitude of UAVs in 2D UAV networks\\
  $W_1$ & Channel of multi-hop mode in 3D UAV networks\\
  $W_2$ & Channel of SCF mode in 3D UAV networks\\
  $W_1^*$ & Channel of multi-hop mode in 2D UAV networks\\
  $W_2^*$ & Channel from central cell to control station\\
  $W_3^*$ & Channel of SCF mode in 2D UAV networks\\
  $\alpha$ & Path loss exponent\\
  $s_n$ & Side length of small cube in 3D UAV networks\\
  $\xi _n$ & Side length of small cell in 2D UAV networks\\
  $R_1$ & Per-hop capacity of 3D UAV networks\\
  $R_2$ & Per-hop capacity of 2D UAV networks\\
  ${\lambda _{SCF}}(n)$ & Per-node capacity of UAV networks with SCF mode\\
  $f( n ) = O(g(n))$ & $\mathop {\lim }\limits_{n \to \infty } \frac{{f(n)}}{{g(n)}} < \infty $\\
  $f\left( n \right) = \Omega (g(n))$ & $g( n ) = O(f(n))$\\
  $f( n ) = \Theta (g(n))$ & $f( n ) = O(g(n))$ and $g( n ) = O(f(n))$\\
  $f\left( n \right) = o(g(n))$ & $\mathop {\lim }\limits_{n \to \infty } \frac{{f(n)}}{{g(n)}} = 0$\\
  $f\left( n \right) = \omega (g(n))$ & $g\left( n \right) = o(f(n))$\\
  $f(n) \equiv g(n)$ & This notation denotes $f( n ) = \Theta (g(n))$\\
  SCF &  Store-carry-and-forward\\
  PDF & Probability density function\\
  \hline
  \hline
 \end{tabular}
 \end{center}
\end{table}

\section{System Model}

The distribution of UAVs relies on the
environment that UAVs are monitoring.
For example, if UAVs are monitoring the
air pollution, they will be distributed in
3D space.
When the UAVs are monitoring the targets on
ground, they may be
distributed in 2D plane.
In the literatures,
the distribution of UAVs is various.
For example,
\cite{UAV_Coverage} assumes
that the UAVs are distributed in
3D space.
While \cite{UAV_Same_Height}
assumes that the UAVs are distributed in
2D plane.
In order to study
the capacity and delay of UAV
networks comprehensively,
3D and 2D UAV networks,
which are illustrated in Fig. \ref{fig_UAV_Net} and Fig. \ref{fig_UAV_Net_2D},
are considered in this paper.

\subsection{Network Model}

\subsubsection{3D UAV Networks}

In Fig. \ref{fig_UAV_Net},
$n$ UAVs are uniformly
distributed in a cube.
The uniformly distributed UAVs
can monitor the environment without priori information.
A control station
is placed at the center of the cube's bottom surface.
The control station is
``responsible for dispatching, coordinating, charging
and collecting of UAVs'' \cite{ZWEI}.
UAVs are regarded as sensors and the
control station is the
sink node collecting the data from all UAVs.
When UAVs are about to exhaust
energy, they will return to the control station
to get energy replenishment.
The returning UAVs are applied to
implement the SCF scheme.
As illustrated
in Fig. \ref{fig_UAV_Net},
the returning UAV C can bring the data of UAV A
to the control station.

The flight velocity and flight duration of UAV are defined as $v$
and $t_0$, respectively\footnote{The values of
$v$ and $t_0$ are assumed to be constant for simplicity. However,
in practice, the values of $v$ and $t_0$ are influenced
by weather, wind, environment, flight attitude, etc.
In this paper, we omit these factors to
yields the capacity and delay
scaling laws of UAV networks.}.
The size of the cube is
$2s \times 2s \times s$,
where $s$ satisfies
the following inequality.
\begin{equation}\label{eq_side_length}
s < \frac{{\sqrt 3 }}{3}\frac{{{t_0}v}}{2}.
\end{equation}

When $s = \frac{{\sqrt 3 }}{3}\frac{{{t_0}v}}{2}$,
the distance between the
vertex P in Fig. \ref{fig_UAV_Net} and the
control station
is $\frac{{{t_0}v}}{2}$,
which is the maximum distance that a UAV can reach
to guarantee that all UAVs can return
within flight duration $t_0$.

Then we consider the process of UAV dispatching.
If UAVs are dispatched simultaneously at a moment,
UAVs will return simultaneously to get energy replenishment,
which brings heavy load to the control station.
If the UAVs are
sequentially dispatched,
they will also return sequentially, which
is beneficial for the control station.
Thus each UAV is dispatched at a time instant
uniformly distributed
in the time interval $[0,{t_0}]$.

\subsubsection{2D UAV Networks}

Similar to \cite{UAV_Same_Height},
UAVs are flying at a given
altitude.
In Fig. \ref{fig_UAV_Net_2D},
$n$ UAVs are uniformly
distributed on an aerial 2D
square plane with side length $2s$ and altitude $h$.
The control station is located on ground.
The SCF mode can be applied when
UAVs are distributed in an aerial 2D plane.
As illustrated in Fig. \ref{fig_UAV_Net_2D},
when a UAV returns to the control station, it
firstly flies to the center of the
plane. Then the returning UAV flies
from the center of
the plane to the control station.
With this kind of returning path, the returning UAVs
can be applied to implement the SCF scheme.
As illustrated in Fig. \ref{fig_UAV_Net_2D},
the returning UAV A can bring
the data of UAV B to the control
station.
The relation between $s$ and $h$ is
\begin{equation}
\sqrt 2 s + h < \frac{{{t_0}v}}{2}.
\end{equation}

Sequential dispatching and returning is also consider
in the 2D UAV networks to
relieve the load of the control station.
Hence each UAV is dispatched at a time instant
uniformly distributed
in the time interval $[0,{t'_0}]$ with ${t'_0} = {t_0} - \frac{2h}{v}$.

\subsection{Network Protocols}

\subsubsection{3D UAV Networks}
\label{subsec_protocol_3D}

In 3D UAV networks,
the multi-hop transmission mode
and the SCF mode coexist to
fully exploit transmission opportunities.
The multi-hop transmission mode
adopts straight-line
routing \cite{XFeng_Scaling_Law_Study,lemma2}.
It is noted that straight-line
routing is widely applied in
scaling law analysis because it can
provide analytical probabilistic model for the distribution of routings.
There are two nonoverlapping
wireless channels $W_1$ and $W_2$
for the multi-hop transmission mode and the SCF mode, respectively.

With the time division multiple access (TDMA) transmission scheme,
the cube in Fig. \ref{fig_UAV_Net} is split
into small cubes.
The side length of the small cubes is
${s_n} = {(c_1\frac{{\log n}}{n})^{1/3}}$,
where ${c_1} = 4{s^3}{c_0}$ with ${c_0} > 0$.
The small cubes are divided to organize
the data transmissions.
Only two UAVs in adjacent small cubes can transmit data to each other.
With straight-line routing,
the one hop transmission distance
is $2\sqrt 3 {s_n}$,
which means that a node can communicate with
the nodes in the adjacent 26 cubes.
With these configurations,
we have the following
lemma.

\begin{Lemm}(\cite{ZWEI, 3D_CRN})\label{lemma_number_of_node}
With $n$ UAVs uniformly
distributed in 3D
space, the number of UAVs
in a small cube with side length $s_n$ is
\begin{equation}
M = \Theta \left( {\log n} \right).
\end{equation}
\end{Lemm}
\begin{proof}
This lemma is the Lemma 1 in \cite{ZWEI}.
\end{proof}

According to Lemma \ref{lemma_number_of_node},
there exist
UAVs in each small cube.
Hence the connectivity of
straight-line routing can be guaranteed.
The 3D UAV networks adopt 27-TDMA scheme \cite{ZWEI, Pan_Li_3D_Capacity_Infocom}.
As illustrated in Fig. \ref{fig_TDMA},
a cluster consists 27 small cubes.
With the entire frame
split into 27 time slots,
the UAVs in each active cube share
one time slot to transmit data.
In addition, the cubes within one cluster
take turns to be active
in a round-robin manner. In such way,
each UAV can take a time slot to transmit data.
The transmit power
is set as ${P_0}s_n^{\alpha}$
to guarantee the
transmission distance of $2\sqrt 3 {s_n}$,
where $\alpha$ is the path loss exponent
and $P_0$ is a constant.
UAV A can transmit data to the UAVs in the
adjacent 26 cubes. The neighborhood of UAV A is defined as
the cluster
consisting 27 cubes
with UAV A in the central cube.
If a returning UAV passes
the neighborhood of UAV A,
it can assist UAV A
to deliver data to the control station.
When multiple UAVs pass
the neighborhood of UAV A simultaneously,
only one returning UAV receives the data
from UAV A.

\begin{figure}[!t]
\centering
\includegraphics[width=0.48\textwidth]{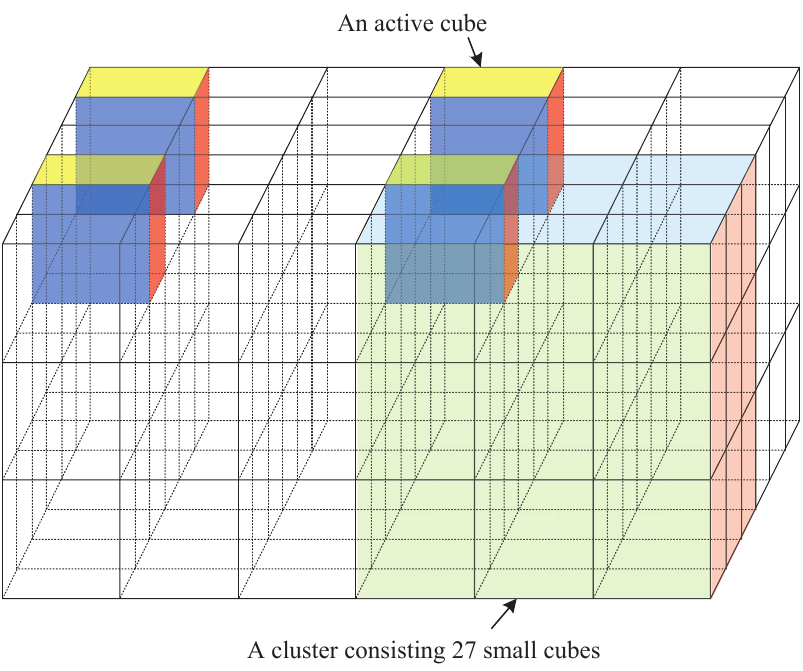}
\caption{27-TDMA scheme.}
\label{fig_TDMA}
\end{figure}

\begin{figure*}[!t]
\centering
\includegraphics[angle=270, width=0.99\textwidth]{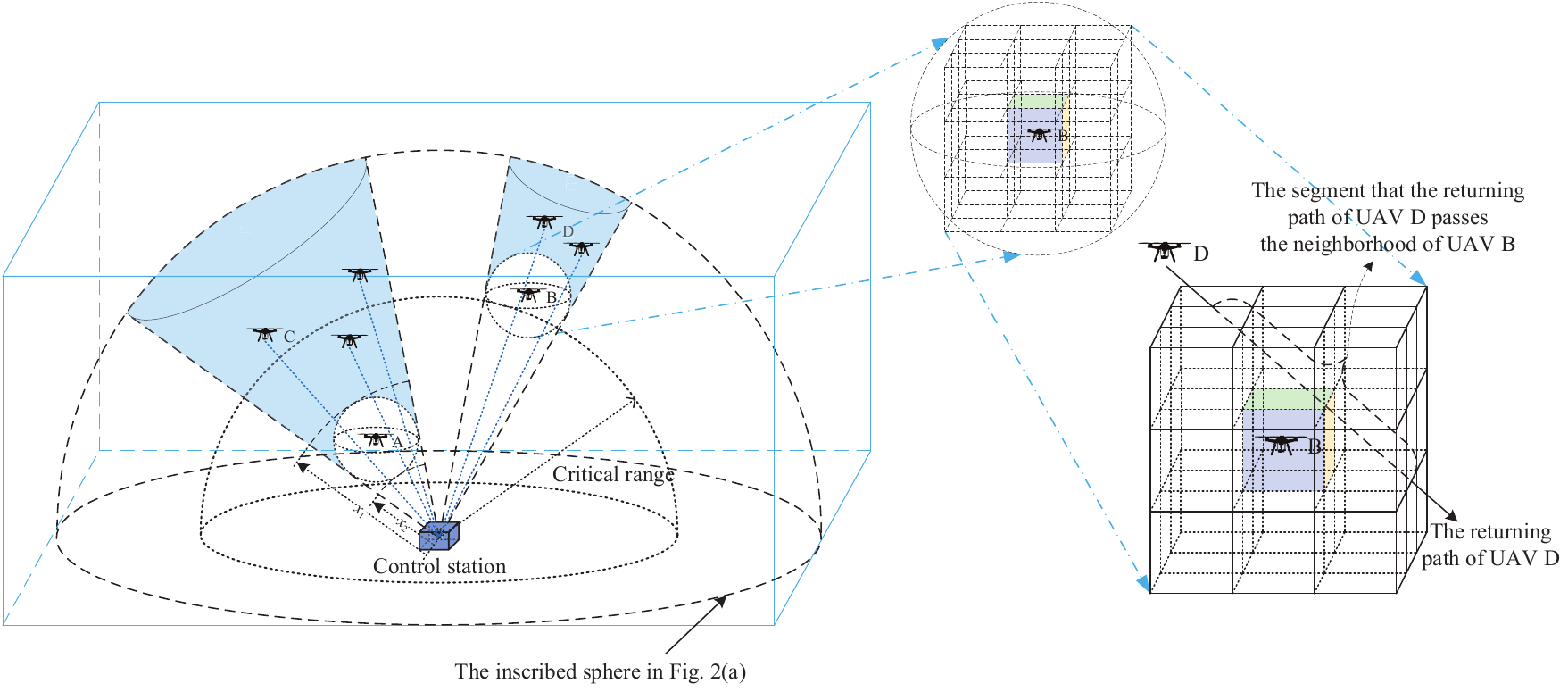}
\caption{UAV networks in 3D space.}
\label{fig_banqiu}
\end{figure*}

\subsubsection{2D UAV Networks}
\label{sec_net_protol_2D}

In 2D UAV networks,
there are three nonoverlapping
channels $W_1^*$, $W_2^*$ and $W_3^*$.
In multi-hop mode,
data is transmitted to the
central cell of the aerial 2D plane
over channel $W_1^*$ with
straight-line routing \cite{Kumar, XFeng_Scaling_Law_Study}.
Then the central cell transmits the data
to the control station over channel $W_2^*$.
In SCF mode, data is transmitted over channel $W_3^*$.
With the TDMA scheme,
the entire square is split into
small cells with side length
${\xi _n} = {(c_3\frac{{\log n}}{n})^{1/2}}$, where
${c_3} = 4{s^2}{c_2}$ with ${c_2} > 0$.
Similar to Lemma \ref{lemma_number_of_node},
a small cell contains $\Theta \left( {\log n} \right)$ UAVs
with high probability (\emph{w.h.p.}).

There are UAVs in each small cell and
the connectivity of straight-line routing is
guaranteed.
The transmit power
is set as ${P_1}{\xi}_n^{\alpha}$ to guarantee the
transmit distance of $2\sqrt 2 {\xi _n}$,
where $P_1$ is a constant.
Hence a UAV can transmit data to another UAV in one of the
adjacent 8 small cells.
We define the cluster
consisting the 9 cells
with UAV B in the central cell
as the neighborhood of UAV B.
Then, if a returning UAV flies through
the neighborhood of UAV B,
it can assist UAV B
to deliver data to the control station.
The 2D UAV networks adopt the 9-TDMA scheme \cite{XFeng_Scaling_Law_Study}\footnote{Letting $M=3$ in Chapter 5 of \cite{XFeng_Scaling_Law_Study}, the multiple access scheme in \cite{XFeng_Scaling_Law_Study}
becomes 9-TDMA.}.

\subsection{Definitions}
The per-node throughput and capacity of a UAV
network are defined as follows.

\begin{Defi}(\cite{XFeng_Scaling_Law_Study}\cite{Pan_Li_3D_Capacity_Infocom})
\label{def_throughput}
Per-node Throughput: The per-node throughput of $\lambda(n)$ bits
per second for a UAV network with $n$ UAVs
is feasible if there exists a scheduling scheme that every UAV
can transmit data to the destination with data rate $\lambda(n)$.
\end{Defi}

The per-node throughput capacity of a UAV network
is defined as follows.

\begin{Defi}(\cite{XFeng_Scaling_Law_Study}\cite{Pan_Li_3D_Capacity_Infocom})
\label{def_capacity}
Per-node Throughput Capacity:
The per-node throughput capacity for a UAV network with $n$ UAVs is the order of
$\Theta \left( {f(n)} \right)$ bits
per second if there are deterministic positive
constants $c < c'$ such that
\begin{equation}
\mathop {\lim }\limits_{n \to \infty } \Pr \left( {\lambda (n) = cf(n){\kern 3pt} {\rm{is}}{\kern 3pt} {\rm{feasible}}} \right) = 1,
\end{equation}
while
\begin{equation}
\mathop {\lim }\limits_{n \to \infty } \Pr \left( {\lambda (n) = c'f(n){\kern 3pt} {\rm{is}}{\kern 3pt} {\rm{feasible}}} \right) < 1.
\end{equation}
\end{Defi}

Without causing confusion, the per-node throughput capacity is
called per-node capacity for short in this paper.

\section{Scaling Laws: 3D UAV Networks}

\subsection{The Capacity of UAV Networks with Multi-Hop Mode}

With multi-hop mode,
all UAVs transmit data to
the control station.
Hence the capacity of multi-hop mode is
determined by the data rate of each UAV.
The data rate of each node in
3D wireless ad hoc networks is as follows.

\begin{Lemm}(\cite{ZWEI, Pan_Li_3D_Capacity_Infocom})\label{lemm_capacity_link_pair}
In 3D wireless ad hoc networks,
the data rate of each node is
\begin{equation}
R_1 \ge \left\{ {\begin{array}{*{20}{c}}
{\frac{1}{{27}} \Theta ({n^{\frac{1}{3}(\alpha  - 3)}})}&{2 < \alpha  < 3}\\
{\frac{1}{{27}} \Theta (\frac{1}{{\log n}})}&{\alpha  = 3}\\
{\frac{1}{{27}} \Theta (1)}&{\alpha  > 3}
\end{array}} \right..
\end{equation}
\end{Lemm}

According to Lemma \ref{lemm_capacity_link_pair},
the capacity of UAV networks with multi-hop mode
is $\Theta \left( {{R_1}} \right)$.

\subsection{The Capacity of UAV Networks with SCF Mode}
\label{sec_finite_t0}

In this section,
the lower bound
of the capacity of UAV networks with SCF mode is studied.
An inscribed hemisphere with radius $L = s$
in Fig. \ref{fig_UAV_Net}
is adopted to estimate
the lower bound of the
number of returning UAVs.
Similarly,
a circumscribed hemisphere of
the cube with radius $L = \sqrt 3 s$ can be used to
derive the upper bound of the capacity of UAV networks with SCF mode.

In Fig. \ref{fig_banqiu},
when the UAVs in the shaded region of
the spherical sector return,
they will carry the data of UAV A to the control station.
To simplify the analysis,
the circumscribed sphere
of the neighborhood of a UAV
is adopted to replace the neighborhood of the UAV.
Fig. \ref{fig_banqiu} shows the
circumscribed sphere of the neighborhood of UAV B.
Denote $x$ as
the distance between
the neighborhood of UAV A and the control station.
Assuming $k$
UAVs are contained
in the shaded region of UAV A,
$k$ is a decreasing function with respect to $x$.
The capacity of UAV networks with SCF mode is shown in the
following theorem.

\begin{Theo}\label{lemm_critical_range}
In 3D UAV networks,
there exist $x_1^*$ and $x_2^*$ as follows.
\begin{equation}\label{eq_critical_range_accurate}
\begin{aligned}
& x_1^* = \frac{1}{3}\frac{{{2^{1/3}}{L^6}{u^2}}}{{{{\left( {27{L^3} - 2{L^9}{u^3} + 3{L^3}\sqrt {3(27 - 4{L^6}{u^3})} } \right)}^{1/3}}}}\\
& + \frac{1}{3}\frac{1}{{{2^{1/3}}}}{\left( {27{L^3} - 2{L^9}{u^3} + 3{L^3}\sqrt {3(27 - 4{L^6}{u^3})} } \right)^{1/3}} \\
& - \frac{1}{3}{L^3}u = \Theta \left( 1 \right),
\end{aligned}
\end{equation}
\begin{equation}\label{eq_x_2_star}
x_2^* = L - \Theta \left( {\frac{1}{{{{(\log n)}^2}}}} \right) + o\left( {\frac{1}{{{{(\log n)}^3}}}} \right),
\end{equation}
where $u = \frac{{16{t_0}v}}{{9\pi {c_1}{c_4}\log n}}\left( {\log \frac{{{t_0}v}}{{c_1^{1/3}{c_4}}} + \frac{1}{3}\log n - \frac{1}{3}\log \log n} \right)$.
$x_1^*$ is defined as the critical range.
When $x \le x_1^*$, the order of the per-node
capacity of 3D UAV networks with SCF mode is
\begin{equation}\label{eq_capacity_7}
{\lambda _{SCF}}(n) = \Theta \left( {\frac{{{R_1}}}{{\log n}}} \right)
\end{equation}
w.h.p.
When  $x \ge x_2^*$, the order of the expectation of
${\lambda _{SCF}}(n)$ is
\begin{equation}\label{eq_clear_form}
E\left[ {{\lambda _{SCF}}(n)} \right] = \Theta \left( {\frac{{1 - {{\left( {\frac{x}{L}} \right)}^3}}}{{{x^2}}}\frac{1}{{{t_0}}}{R_1}} \right),
\end{equation}
where $E[*]$ denotes the expectation of *.
\end{Theo}
\begin{proof}
The proof is provided in Appendix \ref{app_critical_3D}.
\end{proof}

\begin{remark}
The per-node capacity ${{\lambda _{SCF}}(n)}$
is a random variable because of the randomness of the network.
We declare that ${{\lambda _{SCF}}(n)}$ is $\Theta \left( {\frac{{{R_1}}}{{\log n}}} \right)$ w.h.p.
if $\mathop {\lim }\limits_{n \to \infty } \Pr \left\{ {{\lambda _{SCF}}(n) = \Theta \left( {\frac{{{R_1}}}{{\log n}}} \right)} \right\} = 1$,
which is derived from Definition \ref{def_capacity}.
\end{remark}

When $x \le x_1^*$, the order of
the per-node capacity of 3D UAV networks with SCF mode is (\ref{eq_capacity_7}),
which is the data rate of each node in Lemma \ref{lemm_capacity_link_pair}
shared by the UAVs within a small cube.
When $x \le x_1^*$, there are multiple UAVs passing through the neighborhood of the UAV \emph{w.h.p.} at any time.
However, there is only one passing UAV
communicating with the UAV.
Hence the decrease of $x$ within the critical range
will not increase the per-node capacity of UAV network with SCF mode.
When $x \le x_1^*$, the per-node capacity of UAV network with SCF mode is
${{\lambda _{SCF}}(n)} = \Theta \left( {\frac{{{R_1}}}{{\log n}}} \right)$,
which is highest.
However, when $x > x_1^*$,
the highest per-node capacity of $\Theta \left( {\frac{{{R_1}}}{{\log n}}} \right)$
can not be guaranteed.
When $x \ge x_2^*$,
since the time intervals that the returning UAVs flying through the
neighborhood of a UAV are not overlapped, the per-node capacity of
UAV network with SCF mode has clear form.
Hence the $E\left[ {{\lambda _{SCF}}(n)} \right]$ in this case
is provided in (\ref{eq_clear_form}).
When ${x_1} < x < {x_2}$,
the form of the expectation
of $E\left[ {{\lambda _{SCF}}(n)} \right]$ is complex.
The critical range $x_1^*$
is a constant, which can not be eliminated via increasing
the number of UAVs.
Note that the ${{\lambda _{SCF}}(n)}$
within the critical range
is maximum.
The ${{\lambda _{SCF}}(n)}$ outside the
critical range decreases with the increase of $x$.

The capacity of SCF mode
for the UAVs outside the
critical range is smaller than
that within the critical range.
Thus for the UAVs outside the critical range,
the mobility of these UAVs needs to be controlled
to improve the per-node capacity with SCF mode.
In Section \ref{sec_3D_mobility_control},
a mobility control scheme
is designed for capacity improvement.

\subsection{Flight Trajectories Control for Capacity Improvement}
\label{sec_3D_mobility_control}

The per-node capacity of UAV networks with SCF mode
relies on the distance between the UAV and the control station.
When a UAV is outside of the critical range, the capacity with SCF mode
is limited. Hence the flight trajectories of the UAVs
outside the critical range can be controlled to
improve the capacity with SCF mode.
As illustrated in Fig. \ref{fig_UAV_Mobility},
when a UAV is about to return to the control station,
firstly, the UAV flies upward with distance $J/2$.
Secondly, the UAV flies downward with
distance $J$.
Thirdly,
the UAV flies upward with distance $J/2$
and reaches the original spot.
Finally, the UAV returns to
the control station.

\begin{figure}[!t]
\centering
\includegraphics[width=0.38\textwidth]{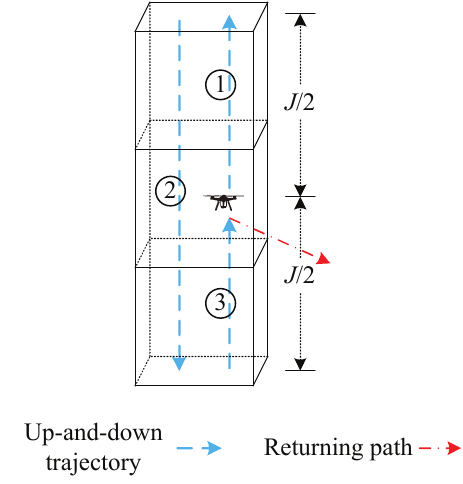}
\caption{The flight trajectory of a UAV with mobility control in 3D UAV networks.}
\label{fig_UAV_Mobility}
\end{figure}

Notice that
the flight trajectories
denoted by the dotted lines in Fig. \ref{fig_UAV_Mobility}
are designed to collect the data of the UAVs along the
flight trajectories,
which are called ``up-and-down trajectories''.
With the ``up-and-down trajectories'',
a UAV outside of the critical range will
encounter more UAVs passing through its neighborhood.
Thus there are more transmission opportunities
with mobility control,
which will improve the capacity of UAV networks with SCF mode.
The capacity with SCF mode for the UAVs outside the critical range
is improved if $J$ is a constant, as in the following theorem.

\begin{Theo}\label{th_mobility_control_1}
The UAVs outside the critical range
fly following the up-and-down trajectories before they return.
If $J$ is a constant, the order of
the per-node capacity
of SCF mode for the UAVs outside the critical range is
\begin{equation}\label{eq_th_UAV_mobility_1_1}
\lambda_{SCF}(n) = \Theta \left( {\frac{{{R_1}}}{{\log n}}} \right).
\end{equation}
\end{Theo}
\begin{proof}
The proof is provided in Appendix \ref{app_mobility_control_1}.
\end{proof}

\begin{remark}
According to Theorem \ref{th_mobility_control_1},
when $J$ is a constant,
all UAVs
enjoy the same per-node capacity scaling laws
of $\Theta ( {\frac{R_1}{{\log n}}} )$ and the critical range is eliminated.
\end{remark}

Although the
up-and-down trajectories improve the capacity of SCF mode,
they reduce the time of environment monitoring for UAVs.
Compared with the scheme without mobility control,
the reduced time for environment
monitoring is $\frac{{2J}}{v}$ per UAV.

\subsection{Delay Analysis}
\label{sec_delay}

The delay of SCF mode is studied in this section.
The following theorem reveals the
impact of the size of the entire region,
the velocity of UAVs, the
number of UAVs and the flight duration
of UAVs on the delay of UAV networks.

\begin{Theo}\label{th_delay}
The delay of UAV networks with SCF mode is
\begin{equation}\label{eq_delay_SCF_3D}
{D_{SCF}}(n) \le \frac{3}{4}\frac{L}{v} + \Theta \left( {{t_0}{{\left( {\frac{{\log n}}{n}} \right)}^{1/3}}} \right).
\end{equation}
\end{Theo}
\begin{proof}
The proof is provided in Appendix \ref{app_delay_3D}.
\end{proof}

\begin{remark}
The delay of SCF mode increases with the increase of $t_0$,
which is due to the fact that when $t_0$ is increased,
the UAVs return infrequently such that
the waiting time for a returning UAV is long.
Moreover, with the increase of $n$,
the waiting time for a returning UAV decreases,
which will reduce the delay of SCF mode.
According to (\ref{eq_delay_SCF_3D}),
the delay of SCF mode is an increasing function of $L$
and the delay of SCF mode is a decreasing function of $v$.
The lower bound of ${D_{SCF}}(n)$ is $\frac{3}{4}\frac{L}{v}$
since the term $\frac{3}{4}\frac{L}{v}$ is
the time that the returning UAV
carries data to the control station.
\end{remark}

\begin{figure*}[!t]
\centering
\includegraphics[width=0.7\textwidth]{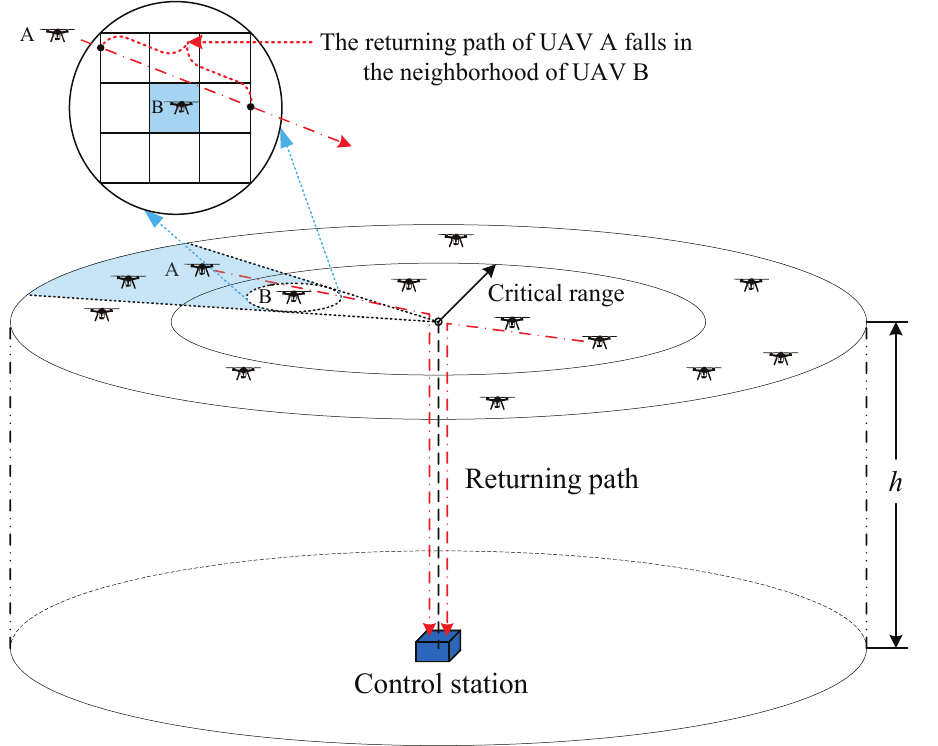}
\caption{UAV networks in 2D space.}
\label{fig_yuan}
\end{figure*}

\section{Scaling Laws: 2D UAV Networks}

\subsection{The Capacity of UAV Networks with Multi-Hop Mode}

In the multi-hop transmission,
all the data will be sent to the central cell.
Then the central cell sends data to the control station.
The per-hop
capacity of the 2D UAV networks
is denoted as $R_2$.
And the capacity from the central cell
to the control station is denoted as $R_3$.
The capacity of UAV networks is constraint
by $\min \left\{ {{R_2},{R_3}} \right\}$.
According to Lemma 6 in \cite{CRN_Capacity_Yin},
$R_2$ is constant on the condition that the path
loss exponent $\alpha > 2$, which is easy to be satisfied.
Besides, $R_3$ is also constant because the central cell
transmits data to the control
station over a separated channel.
The capacity of UAV networks with multi-hop transmission is
\begin{equation}
{\Lambda _{MH}} = \min \left\{ {{R_2},{R_3}} \right\},
\end{equation}
which is a constant.

\subsection{The Capacity of UAV Networks with SCF Mode}

The inscribed sphere with radius $K = s$ in Fig. \ref{fig_yuan}
is adopted to estimate
the lower bound of the
number of returning UAVs
to derive the lower
bound of network capacity with SCF mode.
In Fig. \ref{fig_yuan},
when the UAVs
in the shaded region return,
the data of UAV B can be forwarded by them
to the control station.
Fig. \ref{fig_yuan} illustrates the
circumscribed circle of the
neighbourhood of a UAV.
Define $x$ as the distance between the
circumscribed circle of UAV B's neighborhood
and the center of the plane.
Define $k$ as the number of UAVs in the shaded region.
The capacity of SCF mode is revealed in
the following theorem.

\begin{Theo}\label{lemm_critical_range2}
For 2D UAV networks, there exist $x_1^*$ and $x_2^*$ as follows.
\begin{equation}\label{eq_critical_range_accurate_2D}
x_1^* = \frac{1}{2}\sqrt {4{K^2} + {u^2}}  - \frac{1}{2}u = \Theta \left( 1 \right),
\end{equation}
\begin{equation}\label{eq_x_2_star_2D}
x_2^* = K - \Theta \left( {\frac{1}{{{{(\log n)}^2}}}} \right),
\end{equation}
where $u = \frac{{8{K^2}{t'_0}v}}{{3\sqrt 2 {c_3}{c_6}\log n}}\left( {\log \frac{{{t'_0}v}}{{c_3^{1/2}{c_6}}} + \frac{1}{2}\log n - \frac{1}{2}\log \log n} \right)$. $x_1^*$ is defined as critical range.
When $x \le x_1^*$, the order of the
per-node capacity of 2D UAV networks with SCF mode is
\begin{equation}
{\lambda _{SCF}}(n) = \Theta \left( {\frac{R_2}{{\log n}}} \right)
\end{equation}
w.h.p.
When  $x \ge x_2^*$, the order of the expectation of
${\lambda _{SCF}}(n)$ is
\begin{equation}
E\left[ {{\lambda _{SCF}}(n)} \right] =\Theta \left( {\frac{{\left( {1 - {{\left( {\frac{x}{K}} \right)}^2}} \right)}}{x}\frac{1}{{{t_0}}}{R_2}} \right).
\end{equation}
\end{Theo}
\begin{proof}
The proof is provided in Appendix \ref{app_critical_2D}.
\end{proof}

\begin{remark}
The critical range for 2D UAV networks is a constant, which
can not be eliminated via increasing the value of $n$.
For the UAVs outside the critical range, ${\lambda _{SCF}}(n)$
still decreases with the increase of $x$.
\end{remark}

For 2D UAV networks,
similar mobility control scheme can be designed
to eliminate the critical range and improve the capacity of
UAV networks.

\subsection{Flight Trajectories Control for Capacity Improvement}

For 2D UAV networks,
similar trajectories
are designed for capacity improvement.
As illustrated in Fig. \ref{fig_UAV_Mobility_2D},
when a UAV is about to return, firstly,
the UAV
flies leftward with distance $J/2$.
Secondly, the UAV flies rightward with distance
$J$.
Thirdly, the UAV flies leftward with distance $J/2$
and reaches the original spot.
Finally, the UAV returns to the control station.

The flight trajectories denoted by the dotted
lines in Fig. \ref{fig_UAV_Mobility_2D}
are designed to collect the data of the UAVs
along the flight trajectories, which are called ``left-and-right
trajectories''.
With the left-and-right trajectories,
Theorem \ref{th_mobility_control_1_2D} proves that the capacity of
SCF mode for the UAVs outside the critical
range can be improved if $J$
is a constant.

\begin{figure}[!t]
\centering
\includegraphics[width=0.42\textwidth]{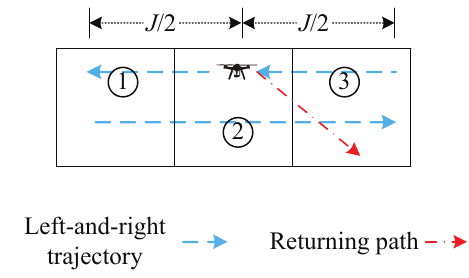}
\caption{The flight trajectory of a UAV with mobility control in 2D UAV networks.}
\label{fig_UAV_Mobility_2D}
\end{figure}

\begin{Theo}\label{th_mobility_control_1_2D}
The UAVs outside the critical range
fly following the left-and-right trajectories before they return.
If $J$ is a constant,
the order of the per-node capacity
of SCF mode for the UAVs that
are outside the critical range is
\begin{equation}\label{eq_th_UAV_mobility_1_1_2D}
\lambda_{SCF}(n) = \Theta \left( {\frac{{{R_2}}}{{\log n}}} \right).
\end{equation}
\end{Theo}
\begin{proof}
The proof is provided in Appendix \ref{app_mobility_control_1_2D}.
\end{proof}

\begin{remark}
With the control of flight trajectories,
all UAVs can achieve the same per-node
capacity of $\Theta \left( {\frac{{{R_2}}}{{\log n}}} \right)$,
which means that the critical range is eliminated.
\end{remark}

\subsection{Delay Analysis}
\label{sec_delay2}

The delay of UAV networks with SCF mode is
revealed in the following theorem.

\begin{Theo}\label{th_delay_2D}
The order of the delay of UAV networks with SCF mode is
\begin{equation}\label{eq_delay_SCF_2D_per}
{D_{SCF}}(n) \le \frac{{2K}}{{3v}} + \frac{h}{v} + \Theta \left( {{t_0}{{\left( {\frac{{\log n}}{n}} \right)}^{1/2}}} \right).
\end{equation}
\end{Theo}
\begin{proof}
The proof is provided in Appendix \ref{app_delay_2D}.
\end{proof}

Overall, the scaling laws for
3D and 2D UAV networks are analyzed.
When the returning UAVs are exploited to
implement the SCF scheme,
the performance of
UAV networks can be improved.

\section{Discussions}

\subsection{Per-node Capacity of UAV Networks with SCF Mode}

The number of returning UAVs has
an impact on the capacity of UAV networks with SCF mode.
In Fig. \ref{fig_cube}, the
control station is located in the origin
and UAVs are distributed in
3D space.
The number of potential
returning UAVs at each location
is denoted by color.
The radius of the neighborhood of each UAV is $r = 0.8$.
Note that
the locations near the control station have a
large number of returning UAVs.
In Fig. \ref{fig_square},
UAVs are distributed in the horizontal 2D plane.
All UAVs fly to the origin and then fly to the
control station.
The radius of the neighborhood of each UAV is $r = 0.5$.
Similar to 3D UAV networks,
with the decrease of the
distance between a UAV and the origin,
the number of returning UAVs
flying through the neighborhood of
the UAV is increasing.

\begin{figure}
\subfigure[{The number of returning UAVs in 3D space.}]{
\label{fig_cube}
\includegraphics[width=0.48\textwidth]{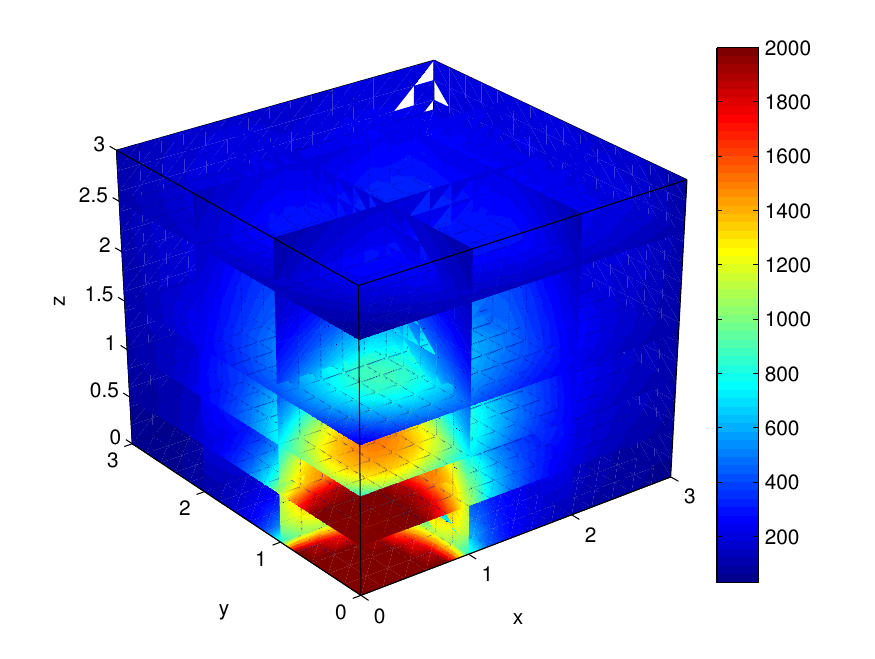}}
\subfigure[{The number of returning UAVs in 2D space.}]{
\label{fig_square}
\includegraphics[width=0.48\textwidth]{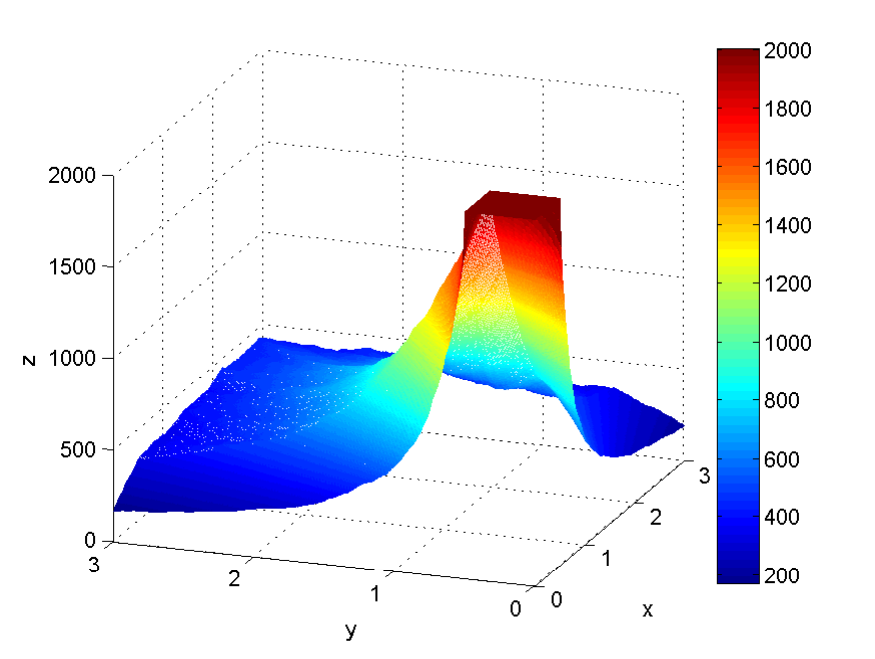}}
\caption{The number of returning UAVs.}
\end{figure}

When the distance between a UAV and the control station $x$
is smaller than the critical range ${x_1^*}$,
$\lambda_{SCF}(n)$ is maximum.
When $x \ge {x_2^*}$,
the expectation of $\lambda_{SCF}(n)$
decreases with the increase of $x$.
When $x \ge x_2^*$, the expectation of ${\lambda _{SCF}}(n)$
in 3D UAV networks
is $\Theta \left( {\frac{{1 - {{\left( {\frac{x}{L}} \right)}^3}}}{{{x^2}}}\frac{1}{{{t_0}}}{R_1}} \right)$.
While the expectation of ${\lambda _{SCF}}(n)$ in 2D UAV networks is
$\Theta \left( {\frac{{\left( {1 - {{\left( {\frac{x}{K}} \right)}^2}} \right)}}{x}\frac{1}{{{t_0}}}{R_2}} \right)$.
The function $g(x)$ is defined as follows.
\begin{equation}
g(x) = \frac{{\left( {1 - {{\left( {\frac{x}{K}} \right)}^2}} \right)}}{x} - \frac{{1 - {{\left( {\frac{x}{L}} \right)}^3}}}{{{x^2}}}.
\end{equation}

In this paper, we assume $K = L$.
It can be proved that $g(x) \ge 0$\footnote{This conclusion can be proved
by finding the derivative of $g(x)$.}.
Hence with the increase of $x$,
the capacity of UAV networks with SCF mode for 2D
UAV networks decreases slower
than that for 3D UAV networks.

\subsection{Impact of Flight Duration}

The $\lambda_{SCF}(n)$ is a decreasing function of
the flight duration $t_0$,
which is due to the fact that when $t_0$ increases, the number
of returning UAVs decreases and $\lambda_{SCF}(n)$
decreases correspondingly.
Besides, the critical range $x_1^*$
decreases with the increase of $t_0$,
which will
decrease the network capacity of UAV networks with
SCF mode.
Note that
the delay is an increasing function of $t_0$.
According to Theorem \ref{th_delay} and Theorem \ref{th_delay_2D},
when $t_0$ increases,
the waiting time for a returning UAV increases correspondingly,
which will increase the delay.

\subsection{Mobility Control}

The mobility of UAVs can be
controlled to eliminate the critical range and
improve the
network capacity with SCF mode. Note
that when $J =  \Theta \left( 1 \right)$,
all UAVs can achieve the maximum capacity,
namely, $\Theta \left( {\frac{{{R_1}}}{{\log n}}} \right)$
for 3D UAV networks and $\Theta \left( {\frac{{{R_2}}}{{\log n}}} \right)$
for 2D UAV networks.
With mobility control,
the capacity of 3D UAV networks with
SCF mode is $\Theta (\frac{{n{R_1}}}{{\log n}})$,
which is $\Theta (\frac{n}{{\log n}})$ times
higher than that with multi-hop mode.
Similarly,
the capacity of 2D UAV networks with SCF mode
is also $\Theta (\frac{n}{{\log n}})$ times
larger than that with multi-hop mode.

However, in the mobility control scheme,
the UAVs switch to the returning state
earlier than the schemes without mobility control.
Hence the
cost of mobility control scheme is that the
time for environment monitoring is reduced.
In this paper, the reduced time for environment monitoring
is $\frac{{2J}}{v}$.

\subsection{Delay of UAV Networks with SCF Mode}

The delay of 3D UAV networks with SCF mode
is illustrated in Fig. \ref{fig_delay}.
Note that
the delay tends to a constant
when $n$ tends to infinity,
which means that the waiting
time for a returning UAV
tends to 0 when $n$ tends to infinity.
The waiting
time for 3D and 2D UAV networks is
upper bounded by $\Theta \left( {{t_0}{{\left( {\frac{{\log n}}{n}} \right)}^{1/3}}} \right)$
and $\Theta \left( {{t_0}{{\left( {\frac{{\log n}}{n}} \right)}^{1/2}}} \right)$
respectively.
Notice that the waiting time for 3D UAV networks is
larger than that for 2D UAV networks.
The increase of $t_0$
will enlarge the waiting time
and increase the delay of SCF mode.
The time that a UAV
flies to the control station
is the lower bound
of the delay of SCF mode.

\begin{figure}[!t]
\centering
\includegraphics[width=0.3\textwidth]{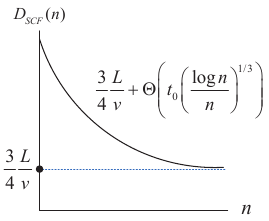}
\caption{The delay of 3D UAV networks with SCF mode.}
\label{fig_delay}
\end{figure}

\section{Conclusion}

In this paper,
the mobility and short flight duration of UAVs
are exploited to improve the performance of UAV networks.
With the observation that UAVs frequently return to
the control station to get energy replenishment,
we exploit the returning UAVs to
implement SCF transmission scheme.
Compared with the performance of multi-hop mode,
The capacity of UAV networks with SCF mode improves
$\Theta (\frac{n}{{\log n}})$ times.
However, a critical range is discovered,
which is a watershed for the capacity of UAV networks.
The per-node capacity of
SCF mode for the UAVs outside the
critical range is smaller than that within the
critical range.
Hence we propose a
mobility control scheme to eliminate the critical range and
improve
the capacity of UAV networks.
This paper has proved the advantages of
SCF mode for
UAV networks,
which may motivate the design of more
SCF transmission schemes for
UAV networks.

\section*{Acknowledgments}
The authors appreciate editor and anonymous reviewers for
their precious time and great effort in improving this paper.

\begin{appendices}

\section{Proof of Theorem \ref{lemm_critical_range}}
\label{app_critical_3D}

As illustrated in
Fig. \ref{fig_banqiu},
a returning UAV
flies through the neighborhood of a
UAV with time
$\frac{{c_1^{1/3}{c_4}}}{v}{\left( {\frac{{\log n}}{n}} \right)^{1/3}}$,
where $c_4$ is a constant and
$0 < {c_4} \le 3\sqrt 3$.
During the time that a
returning UAV flies through
the neighborhood of
a cube,
the amount of data carried by
the returning UAV is
$\frac{{{c_1^{1/3}}{c_4}}}{v}{\left( {\frac{{\log n}}{n}} \right)^{1/3}}R_1$.

Since sequential dispatching is considered,
the $k$ returning UAVs will return uniformly.
We investigate the distribution of the time
interval between two adjacent returning UAVs.
According to order statistics \cite{Order_statistics},
for an ascending sequence of uniformly distributed random variables
${Y_1},{Y_2}, \cdots ,{Y_k}$, the PDF of $W_r = Y_{r+1} - Y_r$ is
provided in Lemma \ref{lemm_range}.
With some manipulations, the
PDF of the time interval between two adjacent returning UAVs
is
\begin{equation}\label{eq_PDF_intervcal}
g(w) = \frac{k}{{{t_0}}}{\left( {1 - \frac{w}{{{t_0}}}} \right)^{k - 1}},
\end{equation}
where $k$ is a random
variable following binomial distribution.
According to Lemma 2 in \cite{lemma2}, when $n$
tends to infinity, the
value of $k$ has the same order with its expectation.

When the time interval $w \le \frac{{c_1^{1/3}{c_4}}}{v}{\left( {\frac{{\log n}}{n}} \right)^{1/3}}$,
multiple UAVs will pass through
the neighborhood of UAV A and
the UAVs in a small cube can share the following capacity
with SCF mode.
\begin{equation}\label{eq_capacity_cube_R1}
{\lambda _{SCF,cube}}(n) = {R_1}.
\end{equation}

Defining
\begin{equation}
{k_{th}} = \frac{{{t_0}}}{{\frac{{c_1^{1/3}{c_4}}}{v}{{\left( {\frac{{\log n}}{n}} \right)}^{1/3}}}},
\end{equation}
we will verify that if $k \ge {k_{th}}\log {k_{th}}$,
(\ref{eq_capacity_cube_R1}) is established.

The probability that $w \le \frac{{c_1^{1/3}{c_4}}}{v}{\left( {\frac{{\log n}}{n}} \right)^{1/3}}$ is

\begin{equation}\label{eq_probab_smaller}
\begin{aligned}
& \Pr \left\{ {w \le \frac{{c_1^{1/3}{c_4}}}{v}{{\left( {\frac{{\log n}}{n}} \right)}^{1/3}}} \right\}\\
& = \int_0^{\frac{{c_1^{1/3}{c_4}}}{v}{{\left( {\frac{{\log n}}{n}} \right)}^{1/3}}} {g(w)dw} \\
& = 1 - {\left( {1 - \frac{1}{{{k_{th}}}}} \right)^k}.
\end{aligned}
\end{equation}

If $k \ge {k_{th}}\log {k_{th}}$, the limit of (\ref{eq_probab_smaller})
with $k_{th}$ tending to infinity is 1.
Hence if $k \ge {k_{th}}\log {k_{th}}$,
${\lambda _{SCF,cube}}(n) = {R_1}$ \emph{w.h.p.}
Similarly, if $k \le \frac{{{k_{th}}}}{{\log {k_{th}}}}$,
the limit of (\ref{eq_probab_smaller})
with $k_{th}$ tending to infinity is 0.
Hence if $k \le \frac{{{k_{th}}}}{{\log {k_{th}}}}$,
there is at most one UAV flying through the neighborhood of UAV A \emph{w.h.p.}
and the expectation of
${\lambda _{SCF,cube}}(n)$ in this case is

\begin{equation}
E\left[ {{\lambda _{SCF,cube}}(n)} \right] = \frac{{c_1^{1/3}{c_4}}}{v}{\left( {\frac{{\log n}}{n}} \right)^{1/3}}{R_1}\frac{k}{{{t_0}}}.
\end{equation}

%When $\frac{{{k_{th}}}}{{\log {k_{th}}}} < k < {k_{th}}\log {k_{th}}$,
%the form of ${\lambda _{SCF,cube}}(n)$ is complex.
%Hence we only investigate the case $k = \Theta \left( {{k_{th}}} \right)$ in this paper.
%
%\begin{equation}
%{\lambda _{SCF,cube}}(n) = \min \left\{ {\frac{{c_1^{1/3}{c_4}}}{v}{{\left( {\frac{{\log n}}{n}} \right)}^{1/3}}{R_1}\frac{k}{{{t_0}}},{R_1}} \right\}.
%\end{equation}

In order to calculate the value of $k$,
we derive the volume of the spherical sector in Fig. \ref{fig_banqiu} as follows.
\begin{equation}
{S_c} = \frac{{2\pi }}{3}{L^3}\left( {1 - \frac{{\sqrt {{x^2} - r_n^2} }}{x}} \right),
\end{equation}
where ${r_n} = \frac{{3\sqrt 3 }}{2}{s_n}$
is the radius of the circumscribed
sphere of a UAV's neighborhood.
In Fig. \ref{fig_banqiu},
the difference in the volume
between the spherical sector with radius $L$ and
the spherical sector with radius $x_2$
is the upper bound of the
volume of shaded region.
Similarly,
the difference in the volume
between the spherical sector with radius $L$ and
the spherical sector with radius $x_1$
is the lower bound of the
volume of shaded region.
With $x_1 = x + r_n$
and $x_2 = x-r_n$, the difference between $x_1$ and $x_2$ can be omitted
in the asymptotic analysis
because $\mathop {\lim }\limits_{n \to \infty } {r_n} = 0$.
Hence the volume of the shaded region is
\begin{equation}\label{eq_vol_upper}
\begin{aligned}
& {S_u} = \frac{{2\pi }}{3}{L^3}\left( {1 - \frac{{\sqrt {{x^2} - r_n^2} }}{x}} \right)\left( {1 - {{\left( {\frac{{x - {r_n}}}{L}} \right)}^3}} \right)\\
& = \frac{{2\pi }}{3}{L^3}\underbrace {\left( {1 - \sqrt {1 - {{\left( {\frac{{{r_n}}}{x}} \right)}^2}} } \right)}_{(a)}\left( {1 - \frac{{{x^3}}}{{{L^3}}}} \right).
\end{aligned}
\end{equation}

Notice that $\frac{{{r_n}}}{x}$
tends to $0$. Using
Taylor series expansion for the term $(a)$ in (\ref{eq_vol_upper}),
we have
\begin{equation}
{S_u} = \frac{{2\pi }}{3}{L^3}\left( {\frac{1}{2}{{\left( {\frac{{{r_n}}}{x}} \right)}^2} + o\left( {{{(\frac{{{r_n}}}{x})}^3}} \right)} \right)\left( {1 - \frac{{{x^3}}}{{{L^3}}}} \right).
\end{equation}

The shaded region of Fig. \ref{fig_banqiu}
averagely contains $k_u$ UAVs.
\begin{equation}\label{eq_ku}
{k_u} = n\frac{{{S_u}}}{{4{L^3}}} = \frac{{\pi n}}{6}\left( {\frac{1}{2}{{(\frac{{{r_n}}}{x})}^2} + o({{(\frac{{{r_n}}}{x})}^3})} \right)(1 - \frac{{{x^3}}}{{{L^3}}}).
\end{equation}

Letting ${k_u} \ge {k_{th}}\log {k_{th}}$, we have
\begin{equation}\label{eq_kthlogkth}
\frac{1}{{{x^2}}}(1 - \frac{{{x^3}}}{{{L^3}}}) \ge u,
\end{equation}
where
\begin{equation}
\begin{aligned}
& u = \frac{{16{t_0}v}}{{9\pi {c_1}{c_4}\log n}}\left( {\log \frac{{{t_0}v}}{{c_1^{1/3}{c_4}}} + \frac{1}{3}\log n - \frac{1}{3}\log \log n} \right)\\
& = \Theta (1).
\end{aligned}
\end{equation}

(\ref{eq_kthlogkth}) is equivalent to
\begin{equation}\label{eq_critical_type1}
\begin{aligned}
& x \le \frac{1}{3}\frac{{{2^{1/3}}{L^6}{u^2}}}{{{{\left( {27{L^3} - 2{L^9}{u^3} + 3{L^3}\sqrt {3(27 - 4{L^6}{u^3})} } \right)}^{1/3}}}}\\
& + \frac{1}{3}\frac{1}{{{2^{1/3}}}}{\left( {27{L^3} - 2{L^9}{u^3} + 3{L^3}\sqrt {3(27 - 4{L^6}{u^3})} } \right)^{1/3}} \\
& - \frac{1}{3}{L^3}u \buildrel \Delta \over = x_1^*.
\end{aligned}
\end{equation}

The $x_1^*$ in (\ref{eq_critical_type1})
is defined as the critical range, which is a constant.
Letting $k \le \frac{{{k_{th}}}}{{\log {k_{th}}}}$, we have

\begin{equation}\label{eq_critical_type22}
\frac{1}{{{x^2}}}(1 - \frac{{{x^3}}}{{{L^3}}}) \le \gamma,
\end{equation}
where
\begin{equation}
\gamma  = \frac{{\frac{{16{t_0}v}}{{9\pi {c_1}{c_4}}}\frac{1}{{\log n}}}}{{\log \frac{{{t_0}v}}{{c_1^{1/3}{c_4}}} + \frac{1}{3}\log n - \frac{1}{3}\log \log n}} = \Theta \left( {\frac{1}{{{{(\log n)}^2}}}} \right).
\end{equation}

Solving the inequality (\ref{eq_critical_type22}),
the following relation can be derived.
\begin{equation}
x \ge L - \Theta \left( {\frac{1}{{{{(\log n)}^2}}}} \right) + o\left( {\frac{1}{{{{(\log n)}^3}}}} \right) \buildrel \Delta \over =  x_2^*.
\end{equation}

The capacity of SCF mode for
the UAVs in
a small cube is summarized as follows.
\begin{equation}
\begin{aligned}
& {{\lambda _{SCF,cube}}(n) = {R_1},}&{x \le x_1^*},\\
& {E\left[ {{\lambda _{SCF,cube}}(n)} \right] = \frac{{c_1^{1/3}{c_4}}}{v}{{\left( {\frac{{\log n}}{n}} \right)}^{1/3}}{R_1}\frac{k}{{{t_0}}},}&{x > x_2^*}.
\end{aligned}
\end{equation}

When $x > x_2^*$, we have
\begin{equation}\label{eq_capacity_within_critical}
E\left[ {{\lambda _{SCF,cube}}(n)} \right] = \Theta \left( {\frac{{1 - {{\left( {\frac{x}{L}} \right)}^3}}}{{{x^2}}}\frac{{\log n}}{{{t_0}}}{R_1}} \right) \le {R_1}.
\end{equation}

According to Lemma \ref{lemma_number_of_node},
$\lambda_{SCF,cube} (n)$ is shared by $\Theta \left( {\log n} \right)$ UAVs.
Thus when $x \le x_1^*$,
the per-node capacity of UAV networks with SCF mode
is ${{\lambda _{SCF}}(n)} = \Theta \left( {\frac{R_1}{{\log n}}} \right)$.
When $x > x_2^*$,
the expectation of ${{\lambda _{SCF}}(n)}$ is
\begin{equation}\label{eq_capacity_SCF_Mode}
E\left[ {{\lambda _{SCF}}(n)} \right] = \frac{{\Theta \left( {\frac{{1 - {{\left( {\frac{x}{L}} \right)}^3}}}{{{x^2}}}\frac{{\log n}}{{{t_0}}}{R_1}} \right)}}{{\Theta (\log n)}}
= \Theta( {\frac{{1 - {{\left( {\frac{x}{L}} \right)}^3}}}{{{x^2}}}\frac{1}{{{t_0}}}{R_1}}).
\end{equation}

\begin{Lemm}[\cite{Order_statistics, New_Zhiqing_WCNC}]
\label{lemm_range}
For an ascending sequence of random variables
${Y_1},{Y_2}, \cdots ,{Y_k}$ uniformly distributed in $[0,1]$,
the PDF of the range of $Y_r$ and $Y_{r+1}$ is
\begin{equation}\label{eq_distribution_01}
{f_{{W_r}}}({w_r}) = \frac{1}{{B(1,k)}}{(1 - {w_r})^{k - 1}},0 \le {w_r} \le 1,
\end{equation}
where
\begin{equation}
B(a,b) = \int_0^1 {{t^{a - 1}}{{(1 - t)}^{b - 1}}dt} ,a > 0,b > 0.
\end{equation}
\end{Lemm}

\section{Proof of Theorem \ref{th_mobility_control_1}}
\label{app_mobility_control_1}

\begin{figure*}[!t]
\centering
\includegraphics[width=0.74\textwidth]{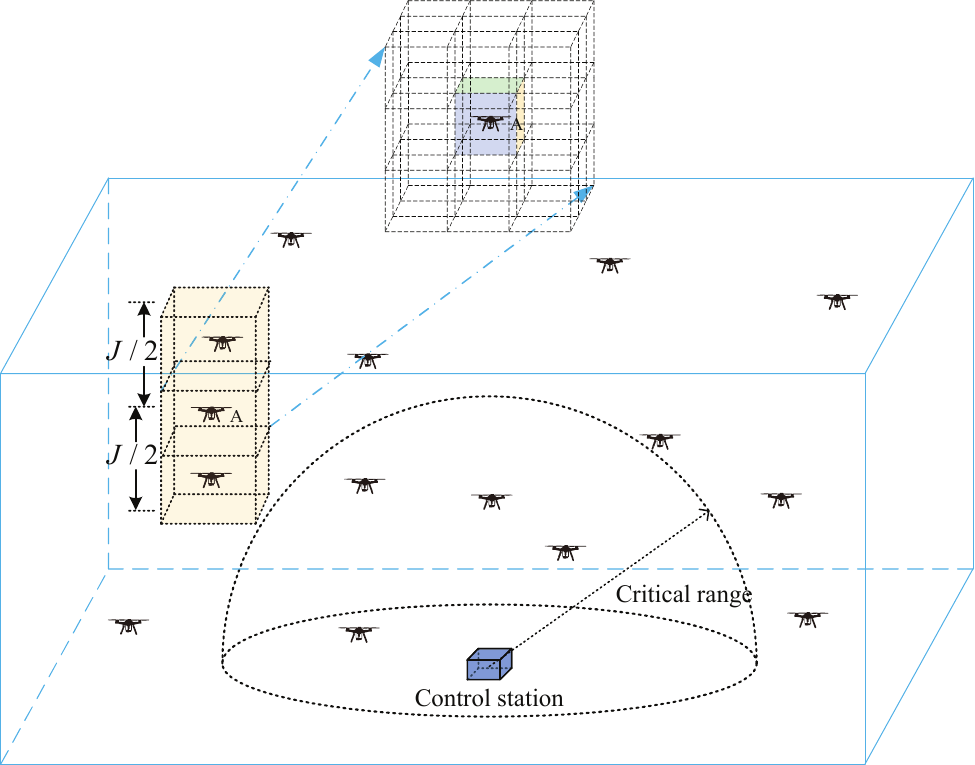}
\caption{The UAVs in the up-and-down trajectories.}
\label{fig_up_and_down}
\end{figure*}

With mobility control,
a UAV will encounter two kinds of
passing UAVs, namely,
the UAVs in up-and-down trajectories
and the UAVs in returning paths.
Theorem \ref{lemm_critical_range}
studies the
capacity of SCF mode contributed by
the UAVs in returning paths.
In this theorem, the capacity of SCF mode contributed by
the UAVs in up-and-down trajectories is studied.
As illustrated in Fig. \ref{fig_up_and_down},
when the UAVs in the shaded cuboid
fly following the up-and-down trajectories,
they will pass through the neighborhood of UAV A.
Hence with mobility control,
the number of potential returning UAVs that fly
through the neighborhood of UAV A
in up-and-down trajectories is
\begin{equation}\label{eq_NJ_for_comp}
N(J) = 9 \frac{{\kappa J}}{{{{\left( {{c_1}\frac{{\log n}}{n}} \right)}^{1/3}}}}{c_5}\log n,
\end{equation}
where $\kappa$ is a constant.

Define
\begin{equation}
{k_{th}} = \frac{{{t_0}}}{{\frac{{3 c_1^{1/3}}}{v}{{\left( {\frac{{\log n}}{n}} \right)}^{1/3}}}}.
\end{equation}

Similar to the
proof of Theorem \ref{lemm_critical_range},
if $N(J) \ge {k_{th}}\log {k_{th}}$,
the amount of data that
can be forwarded for the UAVs in a small cube by the UAVs
in up-and-down trajectories per unit time is ${{\lambda '}_{SCF,cube}}(n) = R_1$ \emph{w.h.p.}
The value of ${k_{th}}\log {k_{th}}$ is
\begin{equation}\label{eq_kth_logkth}
\begin{aligned}
& {k_{th}}\log {k_{th}} = \frac{{{t_0}v}}{{3c_1^{1/3}}}{\left( {\frac{n}{{\log n}}} \right)^{1/3}} \times\\
& \left( {\log \left( {\frac{{{t_0}v}}{{3c_1^{1/3}}}} \right) + \frac{1}{3}\log n - \frac{1}{3}\log \log n} \right).
\end{aligned}
\end{equation}

Comparing (\ref{eq_NJ_for_comp}) with
(\ref{eq_kth_logkth}),
it can be proved that (\ref{eq_NJ_for_comp}) and
(\ref{eq_kth_logkth}) have the same order.
Thus if $J = \Theta (1)$,
${{\lambda '}_{SCF,cube}}(n) = R_1$.
Note that the capacity of the UAVs in a small cube with SCF mode
is the summation
of the ${{\lambda '}_{SCF,cube}}(n)$ and the ${{\lambda}_{SCF,cube}}(n)$
in Theorem \ref{lemm_critical_range},
which is shared by $\Theta \left( {\log n} \right)$ UAVs.
Thus if $J = \Theta (1)$,
the per-node capacity of UAV networks with
SCF mode can be derived
as (\ref{eq_th_UAV_mobility_1_1}).

\section{Proof of Theorem \ref{th_delay}}
\label{app_delay_3D}

The delay of SCF mode consists
the waiting time for a returning
UAV and the time that the returning
UAV carries data
to the control station.
Define the distance between a UAV and the control
station as $X$, which is a random variable.
The probability
that a UAV falls in
the spherical ring with inner radius $x$
and outer radius $x + dx$ is as follows.
\begin{equation}\label{eq_PDF}
\begin{aligned}
& f(x)dx = \Pr \{ x \le X \le x + dx\}\\
& = \frac{{\frac{2}{3}\pi {{(x + dx)}^3} - \frac{2}{3}\pi {{\left( x \right)}^3}}}{{\frac{2}{3}\pi {L^3}}}\\
& = \frac{{{{(x + dx)}^3} - {{\left( x \right)}^3}}}{{{L^3}}} = \frac{{3{x^2}dx}}{{{L^3}}},
\end{aligned}
\end{equation}
where $f(x)$ is the probability density function (PDF)
and the terms $o(dx)$ in (\ref{eq_PDF}) is omitted.
Besides, for the UAV with distance $x$
away from the control station,
the delay to forward its
data to the control station by a returning UAV is ${\frac{x}{v}}$.

Then we study the waiting time
for a returning
UAV, which is the time
interval between two adjacent returning UAVs.
As illustrated in Fig. \ref{fig_banqiu},
$k$ UAVs are located in the shaded region of
the UAV with distance $x$ away from the control station,
which will return sequentially within time $t_0$.
For the UAV with distance $x$ from the control station,
the expectation of the delay with SCF mode
is
\begin{equation}\label{eq_expec_delay}
\begin{aligned}
& {D_{SCF}}(n) = \int_0^L {f(x)\frac{x}{v}dx}  + \int_0^L {f(x){\int_0^{{t_0}} {wg(w)dw} }dx} \\
& = \int_0^L {\frac{{3{x^2}}}{{{L^3}}}\frac{x}{v}dx}  + \int_0^L {\frac{{3{x^2}}}{{{L^3}}}\frac{{{t_0}}}{{k + 1}}dx},
\end{aligned}
\end{equation}
where $g(w)$ is provided in (\ref{eq_PDF_intervcal}).
Replacing the $k$ in (\ref{eq_expec_delay}) with the $k_u$ in (\ref{eq_ku}),
we have
\begin{equation}\label{eq_delay_expand}
\begin{aligned}
& {D_{SCF}}(n) = \int_0^L {\frac{{3{x^2}}}{{{L^3}}}\frac{x}{v}dx}  + \int_0^L {\frac{{3{x^2}}}{{{L^3}}}\frac{{{t_0}}}{{k + 1}}dx}\\
& = \frac{3}{4}\frac{L}{v} + \frac{{36{t_0}}}{{\pi nr_n^2}}\int_0^L {\frac{{{x^4}}}{{{L^3} - {{\left( {x - {r_n}} \right)}^3}}}dx}.
\end{aligned}
\end{equation}

Since we have
\begin{equation}
\begin{aligned}
& {L^3} - {\left( {x - {r_n}} \right)^3}\\
& = (L - x + {r_n})\left( {{L^2} + {{(x - {r_n})}^2} + L(x - {r_n})} \right)\\
& \ge (L - x + {r_n}){L^2},
\end{aligned}
\end{equation}
the following inequality is established.
\begin{equation}\label{eq_delay_expand1}
\begin{aligned}
& {D_{SCF}}(n) \le \frac{3}{4}\frac{L}{v} + \frac{{36{t_0}}}{{\pi nr_n^2{L^2}}}\int_0^L {\frac{{{x^4}}}{{(L - x + {r_n})}}dx} \\
& = \frac{3}{4}\frac{L}{v} + \frac{{36{t_0}}}{{\pi nr_n^2{L^2}}} \times\\
& \left( {\begin{array}{*{20}{l}}
{{{(L + {r_n})}^4}\log \frac{{L + {r_n}}}{{{r_n}}}}\\
{ - \frac{L}{{12}}(25{L^3} + 52{L^2}{r_n} + 42Lr_n^2 + 12r_n^3)}
\end{array}} \right)\\
& \buildrel \Delta \over = {D_{SCF,u}}(n).
\end{aligned}
\end{equation}

With some manipulations,
we have
\begin{equation}\label{eq_SCF_upper_bound}
\begin{aligned}
& {D_{SCF,u}}(n) \equiv \frac{3}{4}\frac{L}{v} + \frac{{36{t_0}}}{{\pi nr_n^2{L^2}}}\left( {L\log \frac{L}{{{r_n}}}} \right)\\
& = \frac{3}{4}\frac{L}{v} + \frac{{36{t_0}}}{{\pi nr_n^2L}}\left( {\log \frac{{2L}}{{3\sqrt 3 c_1^{1/3}}} + \frac{1}{3}\log n - \frac{1}{3}\log \log n} \right)\\
& \equiv \frac{3}{4}\frac{L}{v} + \Theta \left( {{t_0}{{\left( {\frac{{\log n}}{n}} \right)}^{1/3}}} \right).
\end{aligned}
\end{equation}

This theorem is proved.

\section{Proof of Theorem \ref{lemm_critical_range2}}
\label{app_critical_2D}

As illustrated in Fig. \ref{fig_yuan},
a returning UAV flies
through the neighborhood of UAV B with time
$\frac{{c_3^{1/2}{c_6}}}{v}{\left( {\frac{{\log n}}{n}} \right)^{1/2}}$,
where $c_6$ is a constant and $0 < {c_6} \le 3\sqrt 2$.
Defining
\begin{equation}
{k_{th}} = \frac{{{t'_0}}}{{\frac{{c_3^{1/2}{c_6}}}{v}{{\left( {\frac{{\log n}}{n}} \right)}^{1/2}}}},
\end{equation}
where ${t'_0} = {t_0} - \frac{2h}{v}$.
Similar to the proof of Theorem \ref{lemm_critical_range},
if $k \ge {k_{th}}\log {k_{th}}$,
there is at least one returning UAV
flying through the neighborhood of UAV B \emph{w.h.p.}
Then the UAVs in a small cube can share the following
capacity with SCF mode.
\begin{equation}\label{eq_capacity_SCF2}
\lambda_{SCF,cell} (n) = R_2.
\end{equation}

If $k \le \frac{{{k_{th}}}}{{\log {k_{th}}}}$,
there is at most one returning UAV
flying through the neighborhood of UAV B \emph{w.h.p.}
Therefore the expectation of the capacity
for the UAVs in a small cell is
\begin{equation}\label{eq_lambda_SCF2}
E\left[ {{\lambda _{SCF,cell}}(n)} \right] = \frac{{c_3^{1/2}{c_6}}}{v}{\left( {\frac{{\log n}}{n}} \right)^{1/2}}{R_2}\frac{k}{{{t'_0}}}.
\end{equation}

In order to calculate the value of $k$, the area
of the shaded region in Fig. \ref{fig_yuan} is derived.
The area of the sector in Fig. \ref{fig_yuan} with radius $K$ is
\begin{equation}
{S_c} = K^2\arcsin \left( {\frac{{{r_n}}}{x}} \right),
\end{equation}
where ${r_n} = \frac{{3\sqrt 2 }}{2}{\xi_n}$ is
the radius of the circumscribed circle
of UAV B's neighborhood.
The difference of the area between the
sector with radius $K$ and the sector with radius $x + r_n$ is
the lower bound of the area of shaded region.
Similarly, the difference of the area between the
sector with radius $K$ and the sector with radius $x - r_n$ is
the upper bound of the area of shaded region.
Because $\mathop {\lim }\limits_{n \to \infty } {r_n} = 0$,
the term $r_n$ can be omitted and we
adopt the upper bound of
the area of the shaded region as follows.
\begin{equation}\label{eq_vol_upper2}
\begin{aligned}
{S_u} & = \left( {{K^2} - {{(x - {r_n})}^2}} \right)\arcsin \left( {\frac{{{r_n}}}{x}} \right)\\
& \equiv \left( {{K^2} - {x^2}} \right)\arcsin \left( {\frac{{{r_n}}}{x}} \right).
\end{aligned}
\end{equation}

The term $\frac{{{r_n}}}{x}$
tends to $0$ when $n$ tends to infinity.
Using
Taylor series expansion for the
term $\arcsin \left( {\frac{{{r_n}}}{x}} \right)$,
we have
\begin{equation}\label{eq_area_shaded_region}
{S_u} = \left( {{K^2} - {x^2}} \right)\left( {\frac{{{r_n}}}{x} + o\left( {\frac{{{r_n}}}{x}} \right)} \right).
\end{equation}

The shaded region of Fig. \ref{fig_banqiu}
averagely contains $k_u$ UAVs.
\begin{equation}\label{eq_ku2}
{k_u} = n\frac{{{S_u}}}{{4{K^2}}} = \frac{n}{{4{K^2}}}\left( {{K^2} - {x^2}} \right)\left( {\frac{{{r_n}}}{x} + o\left( {\frac{{{r_n}}}{x}} \right)} \right).
\end{equation}

Letting ${k_u} \ge {k_{th}}\log {k_{th}}$, we have
\begin{equation}\label{eq_58}
\frac{1}{x}\left( {{K^2} - {x^2}} \right) \ge \frac{{4{K^2}}}{{n{r_n}}}{k_{th}}\log {k_{th}}.
\end{equation}

Defining
\begin{equation}
u = \frac{{4{K^2}}}{{n{r_n}}}{k_{th}}\log {k_{th}}
\end{equation}
and solving (\ref{eq_58}),
we have
\begin{equation}
x \le \frac{1}{2}\sqrt {4{K^2} + {u^2}}  - \frac{1}{2}u \buildrel \Delta \over = x_1^*.
\end{equation}

The $x_1^*$ is defined as the critical range.
It can be verified that $u = \Theta (1)$.
Thus $x_1^* = \Theta (1)$.
%The case that
%the average capacity of SCF mode for
%the UAVs in a small cell
%is smaller than $R_2$ is investigated.
%\begin{equation}\label{eq_capacity_SCF_inque2}
%\frac{k}{{{t'_0}}}\frac{{{c_3^{1/2}}{c_6}}}{v}{\left( {\frac{{\log n}}{n}} \right)^{1/2}}R_2 \le {R_2}.
%\end{equation}
%
%Replacing the $k$ in (\ref{eq_capacity_SCF_inque2})
%with the $k_u$ in (\ref{eq_ku2}), we have
%\begin{equation}\label{eq_inequality_number_UAV_shaded2}
%\frac{n}{{4{K^2}}}\left( {{K^2} - {x^2}} \right)\left( {\frac{{{r_n}}}{x} + o\left( {\frac{{{r_n}}}{x}} \right)} \right)
%\frac{1}{{{t'_0}}}\frac{{c_3^{1/2}{c_6}}}{v}{\left( {\frac{{\log n}}{n}} \right)^{1/2}} \le 1.
%\end{equation}
%
%With ${c_8} = \frac{{3\sqrt 2 {c_3}{c_6}}}{{8{t'_0}v}}$,
%(\ref{eq_inequality_number_UAV_shaded2}) is equivalent to
%\begin{equation}
%\frac{{{c_8}\log n}}{{{K^2}}}{x^2} + x - {c_8}\log n \ge 0.
%\end{equation}
%
%With some manipulations, we have the inequality
%\begin{equation}\label{eq_x_min_value2}
%x \ge \underbrace {\frac{{K\sqrt {{K^2} + 4c_8^2{{(\log n)}^2}} }}{{2{c_8}\log n}}}_{(d)} - \frac{{{K^2}}}{{2{c_8}\log n}} \buildrel \Delta \over = {x_2^*},
%\end{equation}
%
%The limitation of the term $(d)$ when $n$ tends to infinity is $K$.
Letting $k \le \frac{{{k_{th}}}}{{\log {k_{th}}}}$, we have
\begin{equation}\label{eq_iasuidausi}
\frac{1}{x}\left( {{K^2} - {x^2}} \right) \le \gamma,
\end{equation}
where
\begin{equation}
\gamma  = \frac{{\frac{{8{K^2}}}{{3\sqrt 2 }}\frac{{{{t'}_0}v}}{{{c_3}{c_6}\log n}}}}{{\log \frac{{{{t'}_0}v}}{{c_3^{1/2}{c_6}}} + \frac{1}{2}\log n - \frac{1}{2}\log \log n}} = \Theta \left( {\frac{1}{{{{(\log n)}^2}}}} \right).
\end{equation}

Solving the inequality (\ref{eq_iasuidausi}),
we have
\begin{equation}
x \ge K - \Theta \left( {\frac{1}{{{{(\log n)}^2}}}} \right) \buildrel \Delta \over = x_2^*.
\end{equation}

When $x > x_2^*$, we have
\begin{equation}\label{eq_capacity_within_critical2}
E\left[ {{\lambda _{SCF,cell}}(n)} \right] = \frac{{3\sqrt 2 {c_3}{c_6}}}{{8v}}\frac{{\left( {1 - {{\left( {\frac{x}{K}} \right)}^2}} \right)}}{x}\frac{{\log n}}{{{t_0}}}{R_2} \le {R_2}.
\end{equation}

According to Lemma \ref{lemma_number_of_node},
$\lambda_{SCF,cell} (n)$ is shared by $\Theta \left( {\log n} \right)$ UAVs.
When $x \le x_1^*$, the per-node capacity of UAV networks with
SCF mode
is $\lambda_{SCF} (n) = \Theta \left( {\frac{R_2}{{\log n}}} \right)$.
When $x > x_2^*$, the expectation of $\lambda_{SCF} (n)$ is
\begin{equation}
E\left[ {{\lambda _{SCF}}(n)} \right] = \frac{{{\lambda _{SCF,cell}}(n)}}{{\Theta \left( {\log n} \right)}} = \Theta \left( {\frac{{\left( {1 - {{\left( {\frac{x}{K}} \right)}^2}} \right)}}{x}\frac{1}{{{t'_0}}}{R_2}} \right).
\end{equation}

\section{Proof of Theorem \ref{th_mobility_control_1_2D}}
\label{app_mobility_control_1_2D}

A UAV will encounter two kinds of
passing UAVs with mobility control, namely,
the UAVs in left-and-right trajectories
and the UAVs in returning paths.
The number of potential passing UAVs
in left-and-right trajectories is
\begin{equation}
N(J) = 3 \frac{{\kappa J}}{{{{\left( {{c_3}\frac{{\log n}}{n}} \right)}^{1/2}}}}{c_9}\log n,
\end{equation}
where $\kappa$ is a constant and ${c_9}\log n$ denotes the number
of UAVs within a small cell.
Define
\begin{equation}
{k_{th}} = \frac{{{t_0}}}{{\frac{{3 c_3^{1/2}}}{v}{{\left( {\frac{{\log n}}{n}} \right)}^{1/2}}}}.
\end{equation}

Similar to the proof of Theorem \ref{th_mobility_control_1},
it can be verified that $N(J) = {k_{th}}\log {k_{th}}$ if $J$
is a constant. Thus
Theorem \ref{th_mobility_control_1_2D} is proved.

\section{Proof of Theorem \ref{th_delay_2D}}
\label{app_delay_2D}

Similar to the proof of Theorem \ref{th_delay},
the delay of SCF mode consists
the waiting time for a returning
UAV and the time that the returning
UAV carries data to the control station.
Denote the distance between a UAV
and the center of the plane as $X$.
The probability
that a UAV falls in
the ring with inner radius $x$
and outer radius $x + dx$ is
\begin{equation}\label{eq_probability_fall_in_ring}
\begin{aligned}
& f(x)dx  = \Pr \{ x \le X \le x + dx\}\\
& = \frac{{\pi {{(x + dx)}^2} - \pi {x^2}}}{{\pi {K^2}}} = \frac{{2xdx}}{{{K^2}}},
\end{aligned}
\end{equation}
where the terms $o(dx)$ is omitted
in (\ref{eq_probability_fall_in_ring}).
The time that a UAV flies
to the
control station is $\frac{x}{v} + \frac{h}{v}$.

We investigate the waiting time
for a returning
UAV, which is the time
interval between two adjacent returning UAVs.
The shaded region in Fig. \ref{fig_yuan}
contains $k$ UAVs,
which will return sequentially
within time $t_0$.
The form of the
PDF of the time interval between two adjacent returning UAVs
is the same to (\ref{eq_PDF_intervcal}), where the order of $k$
is provided in (\ref{eq_ku2}).
Thus for the UAV with distance $x$ away
from the center of the plane,
the expectation of the delay is
\begin{equation}\label{eq_expec_delay2}
\begin{aligned}
& {D_{SCF}}(n) = \int_0^K {f(x)\frac{x}{v}dx}  + \int_0^L {f(x){\int_0^{{t_0}} {wg(w)dw} }dx} \\
& + \frac{h}{v} = \int_0^K {\frac{{2x}}{{{K^2}}}\frac{x}{v}dx}  + \int_0^K {\frac{{2x}}{{{K^2}}}\frac{{{t_0}}}{{k + 1}}dx} + \frac{h}{v}.
\end{aligned}
\end{equation}

Substituting the $k_u$ in (\ref{eq_ku2}) into the $k$
in (\ref{eq_expec_delay2}),
we have
\begin{equation}\label{eq_delay_expand2}
\begin{aligned}
& {D_{SCF}}(n) = \int_0^K {\frac{{2x}}{{{K^2}}}\frac{x}{v}dx}  + \frac{h}{v}\\
& + \underbrace {\int_0^K {\frac{{2x}}{{{K^2}}}\frac{{{t_0}}}{{\frac{n}{{4{K^2}}}\left( {{K^2} - {{\left( {x - {r_n}} \right)}^2}} \right)\frac{{{r_n}}}{x}}}dx} }_{(b)},
\end{aligned}
\end{equation}
where the upper bound of the term $(b)$ is
\begin{equation}
\begin{aligned}
(b) & = \int_0^K {\frac{{2x}}{{{K^2}}}\frac{{{t_0}}}{{\frac{n}{{4{K^2}}}\left( {K - x + {r_n}} \right)\left( {K + x - {r_n}} \right)\frac{{{r_n}}}{x}}}dx}\\
& \le \int_0^K {\frac{{2x}}{{{K^2}}}\frac{{{t_0}}}{{\frac{n}{{4{K^2}}}\left( {K - x + {r_n}} \right)K\frac{{{r_n}}}{x}}}dx} \\
& = \underbrace {\frac{{8{t_0}}}{{nK{r_n}}}\left( {{{(K + {r_n})}^2}\log \frac{{K + {r_n}}}{{{r_n}}} - \frac{K}{2}(3K + 2{r_n})} \right)}_{(c)}.
\end{aligned}
\end{equation}

With some manipulations similar to
the derivation of (\ref{eq_SCF_upper_bound}), the term $(c)$
is as follows.
\begin{equation}
(c)  \equiv  \Theta \left( {{t_0}{{\left( {\frac{{\log n}}{n}} \right)}^{1/2}}} \right).
\end{equation}

The delay of SCF mode is derived as (\ref{eq_delay_SCF_2D_per}).
\end{appendices}

\end{document}